\documentclass[letterpaper]{article} 
\usepackage[draft]{aaai25}  
\usepackage{times}  
\usepackage{helvet}  
\usepackage{courier}  
\usepackage[hyphens]{url}  
\usepackage{graphicx} 
\usepackage{natbib}  
\usepackage{caption} 
\usepackage{algorithm}
\usepackage{algorithmic}

\usepackage{newfloat}
\usepackage{listings}
\DeclareCaptionStyle{ruled}{labelfont=normalfont,labelsep=colon,strut=off} 
\floatstyle{ruled}
\newfloat{listing}{tb}{lst}{}
\floatname{listing}{Listing}
\usepackage{amsthm}
\usepackage{amsmath}
\usepackage{tikz}               
\usepackage{amssymb}            
\usepackage{booktabs}
\usepackage{xcolor}

\newtheorem{theorem}{Theorem}
\newtheorem{lemma}{Lemma}

\usepackage[capitalise,noabbrev]{cleveref}
\usepackage{nicefrac}

\usepackage{todonotes}
\let\oldtodo\todo
\renewcommand{\todo}[1]{\oldtodo[inline]{#1}}

\usepackage{pgfplots}
\pgfplotsset{compat=1.18}

\usepackage{transparent}

\pgfdeclarelayer{background}
\pgfsetlayers{background,main}

\usetikzlibrary{positioning}
\tikzstyle{box} = [rectangle, draw, text width=64pt, text centered, rounded corners, minimum height=24pt]
\tikzstyle{arrow} = [thick, ->, >=stealth]
\tikzstyle{dots} = [minimum size=2em, anchor=center, inner sep=0pt, outer sep=0pt, text centered]
\tikzset{
dot/.style = {circle, fill, minimum size=#1,
              inner sep=0pt, outer sep=0pt},
dot/.default = 38pt 
}

\newcommand{\mathfontsize}{\normalsize}

\newcommand{\Ups}{\Upsilon}

\DeclareMathOperator*{\argmax}{arg\,max}

\newcounter{RomanEquationCounter}

\newcommand{\romantag}{\stepcounter{RomanEquationCounter}\tag{\rm\Roman{RomanEquationCounter}}}

\usepackage{mathtools}
\newcommand{\mbar}[3]{%
    \mathrlap{\hspace{#2}\overline{\scalebox{#1}[1]{\phantom{\ensuremath{#3}}}}}\ensuremath{#3}%
}
\newcommand{\xbar}{\mbar{0.72}{0.214em}{X}}
\newcommand{\ybar}{\mbar{0.92}{0.033em}{Y}}
\newcommand{\ttilde}{\widetilde{T}}
\newcommand{\mutilde}{\Tilde{\mu}}
\newcommand{\ktilde}{\widetilde{K}}
\newcommand{\ptilde}{\Tilde{p}}
\newcommand{\jtilde}{\Tilde{J}}

\title{Polynomial Regret Concentration of UCB for Non-Deterministic State Transitions}

\author {
    Can Cömer\textsuperscript{\rm 1}\equalcontrib,
    Jannis Blüml\textsuperscript{\rm 1,2}\equalcontrib,
    Cedric Derstroff\textsuperscript{\rm 1,2}\equalcontrib,
    Kristian Kersting\textsuperscript{\rm 1,2,3,4} 
}
\affiliations {
    \textsuperscript{\rm 1}Computer Science Department, Technical University Darmstadt\\
    \textsuperscript{\rm 2}Hessian Center for Artificial Intelligence (hessian.AI)\\
    \textsuperscript{\rm 3}Centre for Cognitive Science\\
    \textsuperscript{\rm 4}German Research Center for Artificial Intelligence (DFKI)\\
    
    \texttt{jannis.blueml@tu-darmstadt.de, cedric.derstroff@tu-darmstadt.de}
}

\usepackage{bibentry}

\begin{document}

\maketitle

\begin{abstract}


Monte Carlo Tree Search (MCTS) has proven effective in solving decision-making problems in perfect information settings. However, its application to stochastic and imperfect information domains remains limited. This paper extends the theoretical framework of MCTS to stochastic domains by addressing non-deterministic state transitions, where actions lead to probabilistic outcomes. Specifically, building on the work of \citet{shah2020non}, we derive polynomial regret concentration bounds for the Upper Confidence Bound algorithm in multi-armed bandit problems with stochastic transitions, offering improved theoretical guarantees. Our primary contribution is proving that these bounds also apply to non-deterministic environments, ensuring robust performance in stochastic settings. This broadens the applicability of MCTS to real-world decision-making problems with probabilistic outcomes, such as in autonomous systems and financial decision-making. 
\end{abstract}

\section{Introduction}
Monte Carlo Tree Search (MCTS)~\cite{kocsis2006bandit,chaslot2008monte} has become a cornerstone for solving complex decision-making problems in perfect information settings, as demonstrated by its success in games like Chess and Go~\cite{silver2017mastering, silver2018general}. Its ability to efficiently balance exploration and exploitation through the Upper Confidence Bound (UCB) mechanism has made it one of the most effective methods in domains where the environment is fully observable and transitions between states are deterministic. The widespread success of MCTS in perfect information games has led to its adoption in various real-world applications where decision-making based on complete information is critical.

Now imagine running MCTS on the non-deterministic FrozenLake environment; see \cref{fig:exp} and appendix. Here, we ran MCTS with different numbers of simulations.
Can we trust these results? Can we provide guarantees? The issue is that guarantees for running MCTS are only known for deterministic state transitions. We extend them to stochastic domains by addressing non-deterministic state transitions.

\newlength{\mylength}
\setlength{\mylength}{0.3\columnwidth}

\begin{figure}[bt]
    \centering
    \begin{tikzpicture}
        \node (ax) {
        \begin{axis}[
            font=\small,
            height=\mylength,
            width=\mylength,
            xlabel shift={-4pt},
            tick label style={font=\scriptsize},
            grid=major,
            ymin=0.05,
            ymax=0.33,
            ytick={0.05, 0.12, 0.19, 0.26, 0.33},
            scale only axis,
            xmin=0,
            xmax=4,
            xtick={0,...,4},
            xticklabels={$2^{10}$, $2^{11}$, $2^{12}$, $2^{13}$, $2^{14}$},
            ylabel=Discounted Return,
            xlabel=Simulations $n$,
            y tick label style={
                /pgf/number format/.cd,
                    fixed,
                    fixed zerofill,
                    precision=2,
                /tikz/.cd
            },
        ]
            \addplot[blue, mark=*, error bars/.cd, y dir=both,y explicit] coordinates {
                (0, 0.07764059756359422) +- (0,0.24757914540533557 / 17.32)
                (1, 0.09351285311280264) +- (0,0.2671722350262199 / 17.32)
                (2, 0.1419061756726331)  +- (0,0.31487376988354476 / 17.32)
                (3, 0.22554183200445546) +- (0,0.37014178881214516 / 17.32)
                (4, 0.2933227153798441) +- (0, 0.3882443052777321 / 17.32)
            };
        \end{axis}};
        \begin{pgfonlayer}{background}
            \node[xshift=0.23\mylength] at (current axis) {{\transparent{0.5}\includegraphics[width=\mylength]{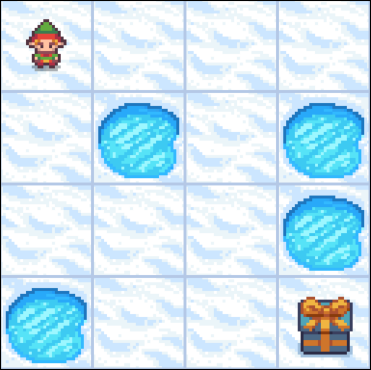}}};
        \end{pgfonlayer}
        \node[minimum height=0.7\mylength] at (1,0){};
    \end{tikzpicture}
    \caption{Running MCTS on $4 \times 4$ FrozenLake: The higher $n$, the better the performance. Is this by accident? No, we establish a polynomial regret concentration for running MCTS on stochastic and imperfect information domains. 
    }
    \label{fig:exp}
\end{figure}
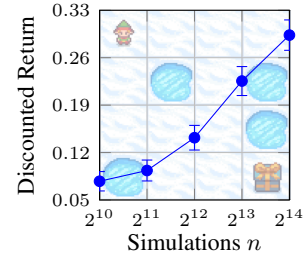

Applying MCTS to stochastic and imperfect information domains remains an ongoing challenge. In such environments, agents must act based on incomplete or uncertain data.
These domains are common in theoretical settings, such as strategic games like poker, and real-world applications, such as autonomous vehicle navigation and financial decision-making. In these scenarios, the environment can behave unpredictably due to probabilistic transitions, making it difficult for traditional MCTS to perform optimally.

Various methods, such as Perfect Information Monte Carlo (PIMC)~\cite{ginsberg2001gib} or Information Set MCTS (ISMCTS)~\cite{cowling2012information}, have been proposed to address these complexities by adapting MCTS to handle imperfect information. These adaptations have proven effective in building strong baselines~\cite{bluml2023alphaze} and even for world champion-level agents like GIB~\cite{Ginsberg99}. Furthermore, recent approaches like ISMCTS-BR~\cite{timbers2022approx} compute approximate best responses in two-player zero-sum games with imperfect information, further expanding the utility of MCTS in these complex settings.
While these methods have demonstrated empirical success, their theoretical underpinnings in non-deterministic environments remain limited. Specifically, there are only a few formal guarantees on their exploration efficiency in stochastic domains, where probabilistic transitions between states introduce additional challenges. Without these guarantees, the performance of MCTS-based methods like ISMCTS or ISMCTS-BR in complex, uncertain environments may suffer from suboptimal exploration and is not guaranteed to converge, hindering their applicability in critical real-world scenarios.

Our contribution addresses this gap by extending the theoretical framework of MCTS to include non-deterministic transitions with fixed probabilities.
Building on the work of \citet{shah2020non}, we derive polynomial regret concentration bounds for UCB in settings where transitions between states are stochastic with fixed probabilities. These bounds ensure regret concentration even in uncertain environments, guaranteeing efficient exploration and enhancing the robustness of MCTS in real-world decision-making problems where probabilistic outcomes are prevalent.


We start by introducing the key concepts before presenting the theoretical framework and analysis of MCTS in stochastic domains with non-deterministic transitions. We provide a sketch of the proof demonstrating polynomial regret concentration for such settings before we discuss practical implications and related work. Finally, the paper concludes with a summary of our contributions.

\section{Preliminaries}

In this section, we introduce the key concepts and algorithms that form the foundation of our work, including the basics of MCTS, the UCB algorithm, and the role of regret concentration in decision-making processes under uncertainty.

\subsection{Multi-Armed Bandit Problems (MABs)}

The core decision-making problem MCTS addresses can be related to the classical Multi-Armed Bandit problem (MAB) \cite{thompson1933likelihood, berry1985bandit, cesa2006prediction}. In an MAB, an agent must choose one arm (action) from a set of possible arms, each associated with an unknown reward distribution. The objective is to maximize cumulative rewards over time by balancing the exploration of less-known arms with exploiting arms that have historically yielded higher rewards.


\paragraph{Upper Confidence Bound:}
At each time step $t$, an arm $I_t\in [K]$ is selected according to
{\mathfontsize{}
    \begin{equation}
    I_{t+1} = \argmax_{i\in[K]} \xbar_{i,s} + \underbrace{\beta^\frac{1}{\xi} t^\frac{\alpha}{\xi}s^{\eta-1}}_{=B_{t,s}}, \label{eq:ucb2}
\end{equation}}%
the \textit{upper confidence bound} of arm $i$ at time $t$ after it was chosen $s = T_i(t)$ times. $\xbar_{i, s}$ is the respective empirical average reward and $\beta, \xi, \eta, \alpha$ are constants. 
$B_{t,s}$ is the exploration bonus.
To avoid division by zero, we define $B_{t,0}=\infty$, which ensures that each arm is explored once first. If more than one maximum exists, one is chosen at random. 

\subsubsection{Application to Trees.}
In a game tree, where nodes represent states and edges represent actions, each tree layer can be treated as a separate MAB; as each MAB is based on the results of the MABs successive layers, complex dependencies arise, making non-leaf layers non-stationary. Successive action selection by applying a UCB variant to every such MAB is called \textit{Upper Confidence bound applied to Trees} (UCT) \cite{kocsis2006bandit}.

\subsection{Monte Carlo Tree Search (MCTS)}

MCTS~\cite{kocsis2006bandit, chaslot2008monte} is an iterative algorithm designed to solve decision-making problems by building and traversing a search tree. Each node represents a state, and each edge corresponds to a possible action. The algorithm comprises four key phases: Selection, Expansion, Simulation, and Backpropagation.

MCTS is highly effective in perfect information settings like Chess and Go, as evidenced by its use in advanced algorithms such as AlphaZero~\cite{silver2018general}. MCTS incrementally refines the estimates of state-action values through repeated simulations, improving decision quality over time.

In MCTS, each state-action pair can be seen as an arm in the MAB framework, where actions are selected using UCT.

\subsection{Regret and Regret Concentration}

In decision-making problems, \textit{regret} measures the difference between the cumulative reward obtained by the optimal policy and the reward achieved by the agent's policy over time. Minimizing regret is crucial for ensuring that the agent's performance improves as more decisions are made.

In MCTS, where each decision node can be modeled as a multi-armed bandit problem, it is important to understand how regret concentrates around its expected value. This is essential for providing performance and convergence guarantees. \citet{shah2020non} demonstrated that UCT with a polynomial exploration bonus (cf. \cref{eq:ucb2}) achieves \textit{polynomial regret concentration} in deterministic environments, where the outcomes of actions are certain.

\section{Theoretical Analysis}

\Citet{shah2020non} demonstrated that UCT with polynomial exploration bonuses leads to polynomial regret concentration in deterministic environments, where transitions between states are fully predictable. However, many real-world decision-making problems, including games with imperfect information and applications involving stochastic systems, necessitate modeling uncertainty in state transitions. In these scenarios, actions can lead to multiple possible outcomes, each with a specific probability, complicating the straightforward application of UCT.

To address this gap, we extend the theoretical guarantees of UCT to environments with non-deterministic but fixed transition probabilities. This extension is vital for domains such as poker, financial modeling, and autonomous systems, where decision-makers must navigate uncertainty due to the non-deterministic nature of the environment.
While \citet{kocsis2006bandit} have claimed that an exponential exploration bonus achieves exponential regret concentration, this only holds for the leaf layer, as it relies on independent and identically distributed reward sequences \cite{auer2002finite}. Recent work has shown that this approach does not maintain exponential regret concentration throughout the search tree \cite{shah2020non, audibert2009exploration}.
We adapt their proof and show that the UCB algorithm can handle non-deterministic transitions effectively. In this context, each action taken at a given state can lead to one of several possible next states, each governed by fixed transition probabilities. We model this situation as a multi-armed bandit problem, where each arm (action) transitions into a set of potential states, each with its own reward distribution. A formal description with necessary assumptions is provided in the next paragraph.




\paragraph{Non-deterministic non-stationary Multi-Armed Bandit Problem:} Let there be $K$ actions the agent can choose from. Each action $i \in [K] = \{1, \dots, K\}$ is connected to $K_i \geq 1$ possible successive states, which are chosen at random. The probability of sampling state $j \in [K_i]$ after choosing action $i\in [K]$ is denoted by $p^i_j\in [0,1]$ with $\sum_{j=1}^{K_i} p^{i}_j = 1, \forall i\in [K]$. Let $X_{i,t}$ denote the reward obtained by selecting action $i\in [K]$ for the $t$-th time. Further let $X^{\left(i\right)}_{j,s}$ be the reward obtained after choosing action $i\in [K]$ and sampling state $j\in [K_i]$ for the $s$-th time. Hence $X_{i,t} = \sum_{j=1}^{K_i} \chi_{\{J^i_t=j\}} X^{\left(i\right)}_{j, T^{\left(i\right)}_j(t)}$, where $\chi \in \{0,1\}$ is the indicator function, $J^i_t\in [K_i]$ is the index of the state that was sampled after action $i$ was chosen and $T^{\left(i\right)}_j(t) = \sum_{s=1}^t \chi_{\{J^i_s=j\}} $ is the number of times $j \in [K_i]$ was sampled after $i\in [K]$ was chosen until time $t \in \mathbb{N}$. Analogously, let $I_t\in [K]$ be the index of the action that was chosen at time $t\in \mathbb{N}$ and $T_i(t) = \sum_{s=1}^t \chi_{\{I_s=i\}} $ the number of times it was chosen until time $t \in \mathbb{N}$. This setup is illustrated in \cref{fig:nondet}. \\
The goal is to find the action with the highest expected average reward in the long run. For now, we assume that such an action exists and denote it as $i*\in[K]$ and that its empirical average converges in expectation to a value that we denote as $\mu_{i*}$. We will later show (\cref{thm:fixedprob}) that this assumption is indeed true. Further, we denote its distance to the converged expected empirical average of the second best action as $\Delta_{min}$, which we also formally define during the proof (cf. \cref{eq:dmin}).

Additionally, we have to make two important assumptions about the underlying MAB.

\paragraph{Assumptions:} First, we assume bounded reward sequences $X^{\left(i\right)}_{j,t} \in [-R, R], \forall i\in [K], j\in [K_i], t \in \mathbb{N}$ and some $R \in \mathbb{R_+}$. Next, we allow the reward sequences to be non-stationary, i.e., their expected value might change over time. Further, we assume the following properties for their empirical averages $\xbar^i_{j,n}=\frac{1}{n} \sum_{t=1}^n X^i_{j,t}$. For every arm $i \in [K]$ we have:
\begin{enumerate}
    \item (\textbf{Convergence}): For every $j \in [K_i]$ there exist $ \mu^i_j \in [-R,R]$ such that
    {\mathfontsize{}
    \begin{equation}
        \mu^i_{j,n}:=\mathop{\mathbb{E}}\left(\xbar^i_{j,n}\right)\xrightarrow{n\xrightarrow{} \infty} \mu^i_j
\end{equation}}%
    \item (\textbf{Concentration}): There exist some constants $\beta>1$, $\xi>0$, $\frac{1}{2} \leq \eta < 1$ such that for every $z \in \mathbb{R}$ with $z\geq 1$ and every $n \in \mathbb{N}$, we have
    {\mathfontsize{}
    \begin{equation}
    \label{eq:concentration}
    \begin{aligned}
         \mathbb{P}\!\left(n\xbar^i_{j,n} - n\mu^i_j \geq n^\eta z\right) \leq \frac{\beta}{z^\xi}, \\
        \mathbb{P}\!\left(n\xbar^i_{j,n} - n\mu^i_j \leq -n^\eta z\right) \leq \frac{\beta}{z^\xi},
    \end{aligned} 
    \end{equation}}%
    which holds for every possible next state $j \in [K_i]$.     
\end{enumerate}

Based on these assumptions, we can now formulate our theoretical main result, establishing polynomial regret concentration for non-deterministic, non-stationary
MABs, which we will then prove.

\begin{figure}[t]
    \centering
    \scalebox{0.9}{
     \begin{tikzpicture}[node distance=8mm, auto,font=\small]
        \node[dots] (rdots){};
        \node[box, above of=rdots](yk){$X^{\left(1\right)}_{K_1,T^{\left(1\right)}_{K_1}\left(T_1(t)\right)}$};
        \node[dots,above=0mm of yk](udots){\rotatebox{90}{\,$\cdots$}};
        \node[box, above=0mm of udots](y2){$X^{\left(1\right)}_{2,T^{\left(1\right)}_2\left(T_1(t)\right)}$};
        \node[box, above of=y2, yshift=0.2cm](y1){$X^{\left(1\right)}_{1,T^{\left(1\right)}_1\left(T_1(t)\right)}$};
        \node[box, below of=rdots](y21){$X^{\left(K\right)}_{1,T^{\left(K\right)}_1\left(T_K(t)\right)}$};
        \node[box, below of=y21, yshift=-0.2cm](y22){$X^{\left(K\right)}_{2,T^{\left(K\right)}_2\left(T_K(t)\right)}$};
        \node[dots, below=0mm of y22](ddots){\rotatebox{90}{\,$\cdots$}};
        \node[box, below=0mm of ddots](y2k){$X^{\left(K\right)}_{K_K,T^{\left(K\right)}_{K_K}\left(T_K(t)\right)}$};

        \node[dots, left=2.5cm of rdots](ldots){\Large\rotatebox{90}{\,$\cdots$}};
        \node[draw, above=of ldots, fill=white, dot](x1){$X_{1,T_1(t)}$};
        \node[draw, below=of ldots, fill=white, dot](xk){$X_{K,T_K(t)}$};
        
        \node[draw, left=of ldots, fill=white, circle](x){$X_t$}; 

        \draw[arrow] (x) -> (x1) node[pos=0.4,above, xshift=-4pt] {$1$};
        \draw[arrow] (x) -> (xk) node[pos=0.4,below, xshift=-4pt] {$K$};

        \draw[arrow] (x1) -> node[pos=0.65, yshift=-4pt] {$p^1_1$} (y1.west);
        \draw[arrow] (x1) -> node[pos=0.6, yshift=-2pt] {$p^1_2$} (y2.west);
        \draw[arrow] (x1) -> node[pos=0.12, yshift=-4pt] {$p^1_{K_1}$} (yk.west);

        \draw[arrow] (xk) -> node[pos=0.7, yshift=-6pt] {$p^K_1$} (y21.west);
        \draw[arrow] (xk) -> node[pos=0.3] {$p^K_2$} (y22.west);
        \draw[arrow] (xk) -> node[pos=0.12, yshift=-5pt] {$p^K_{K_K}$} (y2k.west);
    \end{tikzpicture}}
    \caption{Illustration of non-deterministic, non-stationary Multi-Armed Bandit problem. The upper confidence bound algorithm selects an action within the first layer, and a random transition is applied within the second.}
    \label{fig:nondet}
\end{figure}
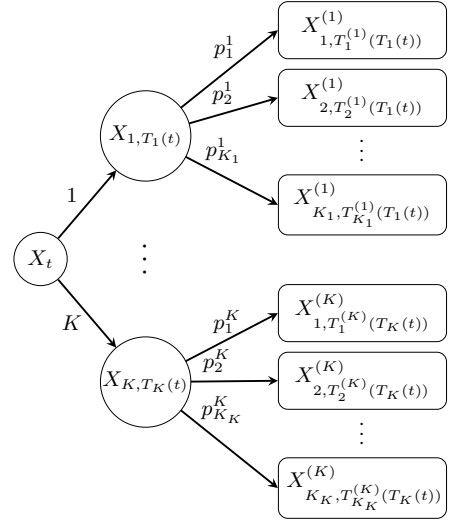

\begin{theorem}
\label{thm:main}
For a non-deterministic, non-stationary MAB satisfying properties 1 (Convergence) and 2 (Concentration), the value
$
    \xbar_n=\frac{1}{n} \sum_{i=1}^K \sum_{j=1}^{K_i} T^i_j(T_i(n)) \xbar^i_{j, T^i_j(T_i(n))}, 
$ 
obtained by applying the presented UCB algorithm with polynomial exploration bonus (\cref{eq:ucb2}) with appropriate parameter $\alpha > 2$ such that $\xi\eta\left(1-\eta\right) \leq \alpha < \xi\left(1-\eta\right)$ and randomly sampling transitions $j$ according to probabilities $\{p^i_j\in [0,1], i\in [K], j\in [K_i]\}$ satisfies properties:
\begin{enumerate}
    \item (Convergence): There exists $\mu_{i*} \in [-R, R]$ such that
    {\mathfontsize{}%
    \begin{equation}
        \mu_{n}:=\mathop{\mathbb{E}}\left(\xbar_n\right)\xrightarrow{n\xrightarrow{} \infty} \mu_{i*}
    \end{equation}}%
\item (Concentration):
    There exist constants $\beta''>1, \xi''>0, \frac{1}{2} \leq \eta'' < 1$ such that $\forall n\in \mathbb{N}, z\in \mathbb{R}$ with $n, z \geq 1$ we have
    {\mathfontsize{}%
    \begin{equation}
        \label{eq:concentration_ND}
        \begin{aligned}
            \mathbb{P}\!\left(n\xbar_n - n\mu_{i*} \geq n^{\eta''}z\right) &\leq \frac{\beta''}{z^{\xi''}}, \\
            \mathbb{P}\!\left(n\xbar_n - n\mu_{i*} \leq -n^{\eta''}z\right) &\leq \frac{\beta''}{z^{\xi''}},
        \end{aligned}
    \end{equation}}%
    where $\eta''=\frac{\alpha}{\xi\left(1-\eta\right)}$, $\xi''=\alpha-1$, and $\beta''$ depends on $R$, $K$, $ \Delta_{\mathit{min}}$, $\beta$, $\xi$, $\alpha$, $ \eta$ and $K_1,\dots,K_K$, see \cref{eq:beta''}. 
\end{enumerate}
\end{theorem}


\section{Proof Sketch of \Cref{thm:main}}
For readability and clarity, we provide only a brief outline of the proof, establishing polynomial regret concentration for a non-deterministic, non-stationary MAB. The full proof and \cref{thm:1,thm:fixedprob} can be found in the appendix. 

\paragraph{Key Idea:}
As illustrated in \cref{fig:nondet}, we divide our non-deterministic non-stationary MAB into two phases, one representing the selection by the UCB algorithm and one representing the random transitions. 

\paragraph{Step 1: Selection Phase}
The goal is to apply the findings of \citet{shah2020non} for deterministic transitions. To this end, we assume that the reward sequences $\{X_{i,t}: t\in \mathbb{N}\}, \forall i\in [K]$ of the intermediate layer already satisfy the convergence and concentration properties for constants $\beta', \xi', \eta'$, from which we deduce that the properties are also satisfied at the root for constants $\beta'', \xi'', \eta''$ (\cref{thm:1}).

\smallskip
\noindent\textbf{Step 2: Transition Phase}\;
The goal is to show that the assumptions of step 1 indeed follow from the assumptions of \cref{thm:main}, i.e., the rewards $\{X^i_{j,t}: j \in [K_i], t\in \mathbb{N}\}, \forall i\in [K]$ satisfy the convergence and concentration properties for constants $\beta, \xi, \eta$. First, by utilizing the independent and identically distributed nature of fixed random transitions, we derive concentration bounds for $T^i_j(n), i\in [K], j\in [K_i]$ against their expectations using Hoeffding's inequality (\cref{lemma}). We utilize these bounds to establish the desired concentration bounds for the rewards $\{X_{i,t}: t\in \mathbb{N}\}, \forall i\in [K]$ for constants $\beta', \xi', \eta'$ (\cref{thm:fixedprob}).

\smallskip
\noindent\textbf{Step 3: Combining Results}\;
Lastly, by combining the concentration results of steps 1 and 2, we obtain that the concentration at the root with constants $\beta'', \xi'', \eta''$ can be deduced by the concentration at the leaves with constants $\beta, \xi, \eta$.

\smallskip
\noindent\textbf{Conclusion:}\;
    The proof establishes that the regret concentrates polynomially around zero, even in non-deterministic environments. This guarantees that the long-term performance of the algorithm approximates the optimal policy with high probability. This concludes the proof.

\section{Practical Implications}
\Cref{thm:main} equips MCTS with convergence guarantees.
In practical terms, this means
when solving a non-deterministic problem with fixed transition probabilities and a fixed horizon, if we can estimate $\Delta_{min}$, we can calculate the exploration constants so that MCTS is guaranteed to return the optimal solution in the limit.
In the next step, we can calculate an $n$ that returns an $\epsilon$-optimal solution with a fixed probability.
This is very powerful in safety-critical domains.

We illustrate this in an empirical experiment by analyzing the regret development over increasing simulations $n$ (cf. \cref{fig:regret}, and the appendix). Shown is the regret of $i*$ of state $14$ (left of the goal) after $n$ simulations ($\mu_{i*} - \xbar_{i*, n}$) for different $n$ averaged over 10 random seeds.
We can observe that the regret not only decreases with higher $n$ but it is also more concentrated, i.e., the standard deviation over the seeds also decreases.
This outcome is expected, given that \cref{thm:main} explicitly addresses this type of behavior.

\section{Related Work}
UCB, introduced by \citet{auer2002finite}, is fundamental in addressing the exploration-exploitation trade-off in MAB problems, offering theoretical guarantees for minimizing cumulative regret. \citet{kocsis2006bandit} extended UCB to MCTS via the UCT algorithm, which has been key to MCTS's success in perfect information settings. Theoretical exploration of MCTS in non-deterministic and stochastic environments has been limited, with recent studies by \citet{dampower} examining its application in such scenarios.
\newpage
Recent advancements like AlphaZero \citep{silver2017mastering, silver2018general} and AlphaFold \cite{alphafold} have successfully combined deep learning with MCTS in perfect information games, but theoretical guarantees for exploration in complex stochastic or imperfect environments remain underexplored. 
The exploration-exploitation dynamics in stochastic MABs have also been studied by \citet{audibert2009exploration} and \citet{BubeckC12}, influencing our adaptation of UCB for probabilistic transitions. Adaptions of MCTS to stochastic or imperfect information settings have been studied, including PIMC \cite{ginsberg2001gib,long2010understanding} and ISMCTS by \citet{cowling2012information}. Yet, these works lack the theoretical guarantees for stochastic domains provided by our work.

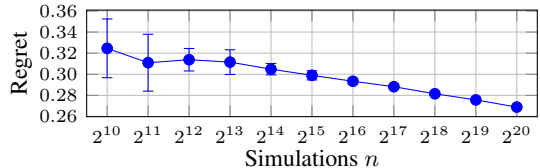
\begin{figure}[t]
    \centering
    \begin{tikzpicture}
        \begin{axis}[
            font=\small,
            height=3cm,
            width=0.9\columnwidth,
            tick label style={font=\scriptsize},
            label shift={-4pt},
            xmin=-0.5,
            xmax=10.5,
            grid=major,
            xtick={0,...,10},
            xticklabels={$2^{10}$, $2^{11}$, $2^{12}$, $2^{13}$, $2^{14}$, $2^{15}$, $2^{16}$, $2^{17}$, $2^{18}$, $2^{19}$, $2^{20}$},
            ylabel=Regret,
            xlabel=Simulations $n$,
            scaled ticks=false,
            yticklabel=\pgfkeys{/pgf/number format/.cd,fixed,precision=2,zerofill}\pgfmathprintnumber{\tick},
        ]
            \addplot[blue, mark=*, error bars/.cd, y dir=both,y explicit] coordinates {
                (0, 0.3245734338948925 ) +- (0, 0.027827663266712675 )
                (1, 0.3109363993393292 ) +- (0, 0.026970077284405657 )
                (2, 0.313770780653483 ) +- (0, 0.010654199982803312 )
                (3, 0.31149801145486294 ) +- (0, 0.011720320922891547 )
                (4, 0.3048168134834497 ) +- (0, 0.005400630848351045 )
                (5, 0.29897585390762416 ) +- (0, 0.004347817018871218 )
                (6, 0.2933685329460737 ) +- (0, 0.00306278166860901 )
                (7, 0.2881853670094512 ) +- (0, 0.0019152723560647764 )
                (8, 0.2815173406549749 ) +- (0, 0.001571483543590802 )
                (9, 0.27571023034531417 ) +- (0, 0.0008869191001255321 )
                (10, 0.2688901616681595 ) +- (0, 0.0006809278228623686 )
            };
        \end{axis}
    \end{tikzpicture}
    \caption{As predicted by \cref{thm:main}, in FrozenLake, the regret decreases and shows greater concentration as $n$ grows.}
    \label{fig:regret}
\end{figure}

\section{Discussion and Conclusion}
This paper extends the theoretical foundations of MCTS and derives polynomial regret concentration bounds for non-deterministic state transitions by building on the work of \citet{shah2020non}.
This ensures that MCTS can explore and operate effectively in stochastic environments, making these algorithms applicable to a broader range of real-world decision-making tasks.
More importantly, this means that given enough resources, MCTS is guaranteed to always converge to the optimal solution in these scenarios.
Thus, this work improves decision-making in uncertain domains by enhancing the reliability and safety of autonomous systems.
However, uncertainty and probability estimates must be handled carefully, especially in high-stakes areas like healthcare.
Theoretical guarantees are an important first step toward building trust in these systems, but rigorous testing and validation remain essential to mitigate risks.
Although our analysis addresses non-deterministic environments, it focuses on fixed transition probabilities.
This assumption may not always hold in real-world scenarios where transition dynamics often evolve over time, limiting the framework's applicability to domains where such assumptions are feasible.
Extending this work to non-stationary environments, where transition probabilities vary over time, will enhance its applicability in domains, such as those involving humans, where systems must adapt in real-time.
Another key direction is empirical validation of these theoretical findings in real-world applications, such as autonomous systems or high-frequency trading, studying the practical implications and reliability in detail.
By improving MCTS' guarantees and ability to handle uncertainty, we hope to inspire further exploration and integration of these concepts into applied AI systems across various domains.

\clearpage
\bibliography{aaai25}
\clearpage

\renewcommand{\mathfontsize}{\scriptsize}
\appendix
\section{List of Symbols}
\subsection{Latin}
\paragraph{$A(t)$} A helper function (cf. \citet{shah2020non}).
\paragraph{$B_{t,s}$} The polynomial exploration bonus for an arm/action chosen $s$ times within $t$ time steps.
\paragraph{$c_1$} A helper constant (cf. \citet{shah2020non}).
\paragraph{$c_2$} A helper constant (cf. \citet{shah2020non}).
\paragraph{$e$} The Euler constant.
\paragraph{$\mathbb{E}$} The expectation.
\paragraph{$f(\bullet)$} A function.
\paragraph{$h$} A layer in the search tree. $h \in \{2,\dots,H\}$. 
\paragraph{$H$} Finite horizon, also used as a superscript regarding the leaf level of the tree. 
\paragraph{$i$} An arm/action.
\paragraph{$i*$} The optimal arm/action.
\paragraph{$I_t$} The selected arm/action at time step $t$.
\paragraph{$j$} A sampled transition/edge.
\paragraph{$J^i_t$} The sampled transition at time step $t$ after action $i\in [K]$ was chosen.
\paragraph{$\jtilde_t$} $= J^i_t$ for some arbitrary fixed $i\in [K]$.
\paragraph{$[k]$} Short notation for $\{1,\dots,k\}$ for some $k\in \mathbb{N}$.
\paragraph{$K$} The number of arms/actions.
\paragraph{$K_i$} The number of possible successive states after arm/action $i\in [K]$ was selected.
\paragraph{$\ktilde$} $= K_i$ for some arbitrary fixed $i\in [K]$.
\paragraph{$n$} number of trials/simulations.
\paragraph{$N_p$} A constant that is a minimum for $t$ (cf. \citet{shah2020non}).
\paragraph{$\mathbb{N}$} The set of natural numbers.
\paragraph{$p^i_j$} The probability of transitioning to state $j \in [K_i]$ after choosing arm/action $i\in [K]$.
\paragraph{$\ptilde_j$} $= p^i_j$ for some arbitrary fixed $i\in [K]$.
\paragraph{$\mathbb{P}$} The probability function.
\paragraph{$R$} The reward bound.
\paragraph{$\mathbb{R}$} The set of real numbers.
\paragraph{$s$} A specific number of selections of an arm/action.
\paragraph{$t$} The current time step.
\paragraph{$T_i(t)$} The number of selections of arm/action $i$ until time step $t$.
\paragraph{$T^i_j(s)$} The number of times transition $j\in [K_i]$ is sampled after arm/action $i\in [K]$ was selected $s$ times.
\paragraph{$\ttilde_j$} $= T^i_j$ for some arbitrary fixed $i\in [K]$.
\paragraph{$X_t$} The reward obtained at time step $t$.
\paragraph{$X_{i,t}$} The reward obtained after selecting arm/action $i$ for the $t-$th time.
\paragraph{$X^i_{j,t}$} The reward obtained after selecting arm/action $i\in [K]$ and sampling transition $j\in [K_i]$ for the $t-$th time.
\paragraph{$\xbar_n$} The empirical mean of $X_t$ over $n$ time steps.
\paragraph{$\xbar_{i,n}$} The empirical mean of $X_{i,t}$ after $i$ was selected $n$ times.
\paragraph{$Y_t$} $= X_{i,t}$ for some arbitrary fixed $i\in [K]$.
\paragraph{$Y_{j,t}$} $= X^i_{j,t}$ for some arbitrary fixed $i\in [K]$.
\paragraph{$\ybar_n$} The empirical mean of $Y_t$ over $n$ time steps.
\paragraph{$\ybar_{j,n}$} The empirical mean of $Y_{j,t}$ after $j$ was sampled $n$ times.
\paragraph{$z$} The concentration parameter with $z\geq 1$.
\subsection{Greek}
\paragraph{$\alpha$} An exploration constant.
\paragraph{$\beta$} A constant at the leaf layer.
\paragraph{$\beta'$} A constant for the intermediate layers.
\paragraph{$\beta''$} A constant at the root.
\paragraph{$\beta_T$} A constant for the concentration of $T^i_j(n)$.
\paragraph{$\delta_{i,n}$} The deviation from the limit $\delta_{i,n}=  \mu_i - \mu_{i,n}  $.
\paragraph{$\Delta_{\mathit{min}}$} $\displaystyle= \min_{i \neq i*}\mu_{i*} - \mu_i$ .
\paragraph{$\gamma$} The discount factor.
\paragraph{$\eta$} A constant at the leaf layer.
\paragraph{$\eta'$} A constant for the intermediate layers.
\paragraph{$\eta''$} A constant at the root.
\paragraph{$\mu_i$} The limit of $\mathbb{E}(\xbar_{i,n})$ for $n \rightarrow \infty$.
\paragraph{$\mu_{i,n}$} The current $\mathbb{E}(\xbar_{i,n})$.
\paragraph{$\mu^i_j$} Limit of $\mathbb{E}(\xbar^i_{j,n})$ for $n \rightarrow \infty$.
\paragraph{$\mutilde_j$} $= \mu^i_j$ for some arbitrary fixed $i\in [K]$.
\paragraph{$\xi$} A constant at the leaf layer.
\paragraph{$\xi'$} A constant for the intermediate layers.
\paragraph{$\xi''$} A constant at the root.
\paragraph{$\Upsilon$} Substitute for $\sum_{i=1}^K T_i(n) \ybar_{i,T_i(n)}$.

\paragraph{$\chi_{\{\bullet\}}$} The indicator function.

%

\clearpage
\section{Proof of \Cref{thm:main}}
As illustrated in the proof sketch, we proceed in two steps; the first covers the selection layer, and the second covers the non-deterministic transition layer. 
\label{sec:proof}

\subsection{Non-stationary MAB with deterministic transitions \cite{shah2020non}}
This subsection focuses on the first layer of our problem. Assume we have a \textbf{non-stationary} MAB, where the reward sequences $\{X_{i,t}\in [-R,R]$, $t\in \mathbb{N}\},\; i\in [K]$ satisfy the following convergence and concentration properties:

\paragraph{Assumptions: Intermediate Layer} The reward sequences are allowed to be non-stationary, i.e., their expected value might change over time, but we assume the following properties for their empirical averages $\xbar_{i,n}=\frac{1}{n} \sum_{t=1}^n X_{i,t}$.
\begin{enumerate}
    \item (\textbf{Convergence}): For every $i \in [K]$ there exist $ \mu_i \in [-R,R]$ such that
    {\mathfontsize{}
    \begin{equation}
        \mu_{i,n}:=\mathop{\mathbb{E}}\left(\xbar_{i,n}\right)\xrightarrow{n\xrightarrow{} \infty} \mu_i
\end{equation}}%
    \item (\textbf{Concentration}): There exist constants $\beta'>1$, $\xi'>0$, $\frac{1}{2} \leq \eta' < 1$ such that for every $z\in \mathop{\mathbb{R}}$ with $z\geq 1$ and $n\in \mathop{\mathbb{N}}$ we have for every $i \in [K]$
    {\mathfontsize{}
    \begin{equation}
    \label{eq:concentration_prime}
    \begin{aligned}
         \mathbb{P}\!\left(n\xbar_{i,n} - n\mu_i \geq n^{\eta'} z\right) \leq \frac{\beta'}{z^{\xi'}}, \\
        \mathbb{P}\!\left(n\xbar_{i,n} - n\mu_i \leq -n^{\eta'} z\right) \leq \frac{\beta'}{z^{\xi'}}.
    \end{aligned} 
    \end{equation}}%
             
\end{enumerate} 
As introduced before, we let $i* \in [K]$ be the action with the highest reward with respect to the converged empirical average reward values, i.e., $i* = \argmax_{i\in[K]}\mu_i$. Without loss of generality, let this action be unique. Additionally, let $\delta_{i,n}:= \mu_{i,n}-\mu_i, i\in[K], n\in\mathbb{N}$ be the difference between the currently expected reward of action $i\in[K]$ and its converged value and
\begin{equation}\label{eq:dmin}
    \Delta_{min} = \min_{i \in [K] \setminus \{i*\}} \mu_{i*} - \mu_{i}
\end{equation}
the distance between the optimal reward and the reward obtained by the second best action. 

\begin{theorem}
\label{thm:1}
For a, possibly non-stationary, MAB satisfying the convergence and concentration property, the value $\xbar_n=\frac{1}{n} \sum_{i=1}^K T_i(n) \xbar_{i, T_i(n)}$ obtained by the presented UCB algorithm (cf. \cref{eq:ucb2}) with appropriate parameter $\alpha > 2$ such that $\xi'\eta'\left(1-\eta'\right) \leq \alpha < \xi'\left(1-\eta'\right)$ satisfies properties:
\begin{enumerate}
    \item (Convergence):
    {\mathfontsize{}
    \begin{equation}
    \begin{aligned}
        &\left|\mathop{\mathbb{E}}\left(\xbar_n - \mu_{i*}\right)\right|
        \leq \left|\delta_{i*,n} + \frac{1}{n} \left(2R\left(K-1\right)\right) \times \right.\\
        &\left.\left(\frac{2}{\Delta_{\mathit{min}}\beta'^{\frac{1}{\xi'}}\left(1-\eta'\right)} \eta\prime^{\frac{\alpha}{\xi'\left(1-\eta'\right)}} + \frac{2}{\alpha-2} + 1\right)\right|, \\ 
        &\text{i.e.,} \lim_{n\xrightarrow{}\infty} \mathop{\mathbb{E}}\left(\xbar_n\right) = \mu_{i*}.
    \end{aligned}
    \end{equation}}%
    \item (Concentration):
    There exist constants $\beta''>1, \xi''>0, \frac{1}{2} \leq \eta'' < 1$ such that $\forall n\in \mathbb{N}, z\in \mathbb{R}$ with $n,z \geq 1$ we have
    {\mathfontsize{}
    \begin{equation}
    \begin{aligned}
        \mathbb{P}\!\left(n\xbar_n - n\mu_{i*} \geq n^{\eta''}z\right) \leq \frac{\beta''}{z^{\xi''}}, \\
        \mathbb{P}\!\left(n\xbar_n - n\mu_{i*} \leq -n^{\eta''}z\right) \leq \frac{\beta''}{z^{\xi''}},
    \end{aligned}
    \end{equation}}%
    where $\eta'' = \frac{\alpha}{\xi'\left(1-\eta'\right)}$, $\xi''=\alpha-1$ and $\beta''$ depends on $R, K, \Delta_{\mathit{min}}, \beta', \xi', \alpha, \eta'$ (cf. \citet{shah2020non}, and \cref{eq:beta''}).
\end{enumerate}
\end{theorem}
\begin{proof}
    See \citet{shah2020non}.
\end{proof}

\subsubsection{Non-deterministic Transitions with Fixed Probabilities}
In this subsection, we observe the second layer of our non-deterministic, non-stationary MAB, representing the random transitions. To improve readability, we fix some arbitrary $i\in[K]$ and define $s:=T_i(t)\in \mathbb{N}$, $Y_s:=X_{i,s}$,  $\ktilde := K_i$ and for $j\in [\ktilde]$ we let $Y_{j, \ttilde_j(s)} := X^i_{j, T^i_j(s)}$, $\ttilde_j(s) := T_j^i(s),$ $\mutilde_j:=\mu_j^i$, $\ptilde_j:=p_j^i,$ $\jtilde_t := J_t^i$. This setup is illustrated in \cref{fig:nondetonly}.


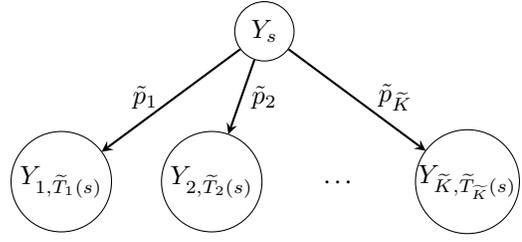
\begin{figure}[tb]
    \centering
    \begin{tikzpicture}[node distance = 2cm,
                        auto]
        \tikzstyle{arrow} = [thick,->,>=stealth]
        \node[draw, fill=white, circle](x){$Y_s$}; 
        \node[draw, below of=x, xshift=-0.7cm, fill=white, dot](x2){$Y_{2,\ttilde_2(s)}$};
        \node[draw, left of=x2, fill=white, dot](x1){$Y_{1,\ttilde_1(s)}$};
        \node[right of=x2, xshift=-0.3cm](dots){\dots};
        \node[draw, right of=dots, xshift=-0.3cm, fill=white, dot](xk){$Y_{\ktilde,\ttilde_{{\ktilde}}(s)}$};
    
        \draw[arrow] (x) -> (x1);
        \draw[arrow] (x) -> (x2);
        \draw[arrow] (x) -> (xk);

        \node[below=0.2cm of x](p2){$\ptilde_2$};
        \node[left=1cm of p2](p1){$\ptilde_1$};
        \node[right=1.1cm of p2](pk){$\ptilde_{\ktilde}$};
    \end{tikzpicture}
    \caption{A non-deterministic transition layer where the agent transitions to state $j \in [\ktilde]$ with probability $\ptilde_j$ to receive reward $Y_{j,\ttilde_j(s)}$.}
    \label{fig:nondetonly}
\end{figure}



Note that now $\{\chi_{\{\jtilde_t = j\}}\colon t\geq 1\}$ are independent and identically distributed for $j\in[\ktilde]$, since the transition probabilities are fixed, and $\mathbb{P}\!\left(\jtilde_t=j\right) = \ptilde_j, \forall t\in \mathbb{N}$.

\paragraph{Assumptions: Leaf Layer}
Again, we allow the reward sequences $\{Y_{j,t}\colon t\geq 1\},\; j\in [\ktilde]$ to be non-stationary, but they fulfill the following properties, that we already know from \cref{thm:main}:

\begin{enumerate}
    \item (\textbf{Convergence}): For every $j \in [\ktilde]$ there exist $ \mutilde_j \in [-R,R]$ such that
    {\mathfontsize{}
    \begin{equation}
        \mutilde_{j,n}:=\mathop{\mathbb{E}}\left(\ybar_{j,n}\right)\xrightarrow{n\xrightarrow{} \infty} \mutilde_j
\end{equation}}%
    \item (\textbf{Concentration}): There exist constants $\beta>1$, $\xi>0$, $\frac{1}{2} \leq \eta < 1$ such that for every $z\in \mathop{\mathbb{R}}$ with $z\geq 1, n\in \mathbb{N}$ and $j \in [\ktilde]$, we have
    {\mathfontsize{}
    \begin{equation}
    \begin{aligned}
        \mathbb{P}\!\left(n\ybar_{j,n} - n\mutilde_j \geq n^\eta z\right) \leq \frac{\beta}{z^\xi}, \\
        \mathbb{P}\!\left(n\ybar_{j,n} - n\mutilde_j \leq -n^\eta z\right) \leq \frac{\beta}{z^\xi}.
    \end{aligned}
    \end{equation}}%
\end{enumerate}

\begin{theorem}
\label{thm:fixedprob}
For a layer with, possibly non-stationary, reward sequences $\{Y_{j,t}\colon t\geq 1\}, j\in [\ktilde]$ satisfying properties $1$ and $2$, the value $\ybar_n=\frac{1}{n} \sum_{i=1}^{\ktilde} \ttilde_j(n) \ybar_{j,\ttilde_j(n)}$ obtained by randomly sampling edges based on the fixed probabilities $\ptilde_1,\dots,\ptilde_{\ktilde} \in [0,1]$ with $\sum_{j=1}^{\ktilde} \ptilde_j = 1$ satisfies properties:
\begin{enumerate}
    \item (Convergence): For $n \xrightarrow{} \infty$ we have
    {\mathfontsize{}
    \begin{equation}
        \mathop{\mathbb{E}}\left(\ybar_n\right) \xrightarrow{} \mutilde:= \sum_{j=1}^{\ktilde} \ptilde_j \mutilde_j
\end{equation}}%
    \item (Concentration): There exist constants $\beta'>1, \xi'>0, \frac{1}{2} \leq \eta' < 1$ such that $\forall n\in \mathbb{N}, z\in R$ with $n,z \geq 1$ we have
    {\mathfontsize{}
    \begin{equation}
    \begin{aligned}
    \mathbb{P}\!\left(n\ybar_n - n\mu \geq n^{\eta'}z\right) \leq \frac{\beta'}{z^{\xi'}},\\
    \mathbb{P}\!\left(n\ybar_n - n\mu \leq -n^{\eta'}z\right) \leq \frac{\beta'}{z^{\xi'}},
    \end{aligned} \label{eq:concy}
    \end{equation}}%
    where $\xi' = \xi, \eta'=\eta$ and $\beta' = \beta 2^{\xi+2}R^{\xi}\ktilde^{\xi+1}$.
\end{enumerate}%
\end{theorem}%
\noindent Before we prove this theorem, note the following lemma.

\begin{lemma}
\label{lemma}
For every $\eta \in \left[\nicefrac{1}{2}, 1\right)$ and any finite $\xi>0$ there exist $\beta_T>1$ large enough such that $\forall n\in \mathbb{N}, z\in \mathbb{R}$ with $n,z \geq 1$ and $j\in [\ktilde]$ we have

{\mathfontsize{}
\begin{equation}
\begin{aligned}
    \mathbb{P}\!\left(\ttilde_j(n) - n \ptilde_j \geq n^{\eta}z\right) \leq \frac{\beta_T}{z^{\xi}}, \\
    \mathbb{P}\!\left(\ttilde_j(n) - n \ptilde_j \leq -n^{\eta}z\right) \leq \frac{\beta_T}{z^{\xi}}.
\end{aligned}
\end{equation}}%

\end{lemma}
\begin{proof} 
As a starting point, we observe that
{\mathfontsize{}
\begin{equation}
\begin{aligned}    
    \mathop{\mathbb{E}}\left(\ttilde_j(n)\right) & = \mathop{\mathbb{E}}\left(\sum_{t=1}^n \chi_{\{\jtilde_t=j\}}\right) \\
    &= \sum_{t=1}^n \mathop{\mathbb{E}}\left(\chi_{\{\jtilde_t=j\}}\right) \\
    &= \sum_{t=1}^n \mathbb{P}\!\left(\jtilde_t=j\right) \\
    &= n\ptilde_j.
\end{aligned}
\end{equation}}%
Then, we can deduce exponential concentration by applying Hoeffdings inequality, which we will weaken into a polynomial concentration of the desired form. We have
{\mathfontsize{}
\begin{equation}
\begin{aligned}
    &\mathbb{P}\!\left(\ttilde_j(n) - n \ptilde_j \geq n^{\eta}z\right) \\
    &\leq e^{-2 \frac{n^{2\eta} z^2}{n}}   && (\text{Hoeffdings ineq.})  \\
    &= \frac{1}{ e^{2n^{2\eta-1} z^2}} \\
    &\leq \frac{1}{ e^{2 z^2}  }   && (2\eta -1 \geq 0) \\  
    &\leq  \frac{\beta_T}{ z^{\xi}  }. && \left(\forall \beta_T \geq \max_{z \geq 1}\frac{z^{\xi}}{e^{2z^2}}\right)
\end{aligned}
\end{equation}}%
Hence, for any finite $\xi > 0$, we can choose $\beta_T$ large enough, such that the desired inequalities hold. Specifically any
\begin{equation}
    \beta_T \geq \xi^{\frac{\xi}{2}} 2^{-\xi} e^{-\frac{\xi}{2}}   
\end{equation}
will suffice.
The other direction
follows analogously.\\
Note that for large enough $z$ any $\beta_T>1$ suffices, since $z^{\xi} = \mathcal{O}(e^{2z^2})$ for any finite $\xi>0$.
\end{proof}

\begin{proof}\textit{(\Cref{thm:fixedprob})}
From \cref{lemma}, we can deduce for $j\in [\ktilde]$ that there exists some $\beta_T$ such that
{\mathfontsize{}
\begin{equation}
\mathbb{P}\!\left(\ttilde_j(n) - n\ptilde_j \geq \frac{n^\eta z}{2R\ktilde}\right) \leq \frac{\beta_T}{\left(\frac{z}{2RK}\right)^\xi} = \frac{\beta_T 2^\xi R^\xi K^\xi}{z^\xi},   
\end{equation}}%
\noindent which for $z \geq 2RK$ directly follows from \cref{lemma} and for $z \leq 2RK$ trivially follows from the right hand side being greater than $1$ since $\frac{\left(2RK\right)^\xi}{z^\xi} \geq 1$ and $\beta_T >1 $.

Additionally, we can assume without loss of generality that we can choose the same constants for the bounds of $\ttilde_j(n)$ and $\ybar_{j,n}$, i.e., $\beta = \beta_T$ since we can choose the larger one and both bounds will hold. We will, therefore, omit the subscript $T$ to improve readability and use only $\beta$.

Next, we show the \textbf{concentration property} of \cref{thm:fixedprob}. Due to space constraints, we introduce the new variable $\Ups = \sum_{j=1}^{\ktilde} \ttilde_j(n) \ybar_{j,\ttilde_j(n)}$. We start with the left-hand side of \cref{eq:concy}. To obtain 2 new terms that we can treat separately, we transform it as follows:
{\mathfontsize{}
\begin{align*}
    \text{conc.} &= \mathbb{P}\!\left(n\ybar_n - n\mutilde \geq n^{\eta}z\right) \\
    &= \mathbb{P}\!\left(n \frac{1}{n} \Ups - n \sum_{j=1}^{\ktilde} \ptilde_j \mutilde_j \geq n^{\eta}z\right)  \\
    &\leq \mathbb{P}\!\left(
    \begin{aligned}
       &\left\{\Ups - n \sum_{j=1}^{\ktilde} \ptilde_j \mutilde_j \geq n^{\eta}z \right\} \\
       &\cap \bigcup_{j=1}^{\ktilde} \left\{\left|\ttilde_j(n)-n\ptilde_j\right|\geq \frac{n^\eta z}{2R\ktilde}\right\}
    \end{aligned}
    \right)  \romantag\label{i} \\
    &\phantom{=} \; + \mathbb{P}\!\left(
    \begin{aligned}
        \left\{\Ups - n \sum_{j=1}^{\ktilde} \ptilde_j \mutilde_j \geq n^{\eta}z\right\} \\
        \cap \bigcap_{j=1}^{\ktilde} \left\{\left|\ttilde_j(n)-n\ptilde_j\right|\leq \frac{n^\eta z}{2R\ktilde}\right\}
    \end{aligned}
    \right),  \romantag\label{ii}
\end{align*}}%
where the last inequality follows from the union bound. 
Now, we bound each term separately. For (\ref{i}) again the union bound and \cref{lemma} yield
{\mathfontsize{}
\begin{equation}
\begin{aligned}
    (\ref{i}) &=\mathbb{P}\!\left(
    \begin{aligned}
       &\left\{\Ups - n \sum_{j=1}^{\ktilde} \ptilde_j \mutilde_j \geq n^{\eta}z \right\} \\
       &\cap \bigcup_{j=1}^{\ktilde} \left\{\left|\ttilde_j(n)-n\ptilde_j\right|\geq \frac{n^\eta z}{2R\ktilde}\right\}
    \end{aligned}
    \right) \\
    &\leq \sum_{j=1}^{\ktilde} \mathbb{P}\!\left(
    \begin{aligned}
        &\left\{\Ups - n \sum_{j=1}^{\ktilde} \ptilde_j \mutilde_j \geq n^{\eta}z \right\} \\
        &\cap \left\{\left|\ttilde_j(n)-n\ptilde_j\right|\geq \frac{n^\eta z}{2R\ktilde}\right\}
    \end{aligned}
    \right)  \\
    &\leq \sum_{j=1}^{\ktilde} \mathbb{P}\!\left( \left\{\left|\ttilde_j(n)-n\ptilde_j\right|\geq \frac{n^\eta z}{2R\ktilde}\right\}\right)  \\
    &\leq 2\ktilde \frac{\beta 2^\xi R^\xi \ktilde^\xi}{z^\xi}.
\end{aligned}
\end{equation}}%
Before we establish the bound for (\ref{ii}), first note the following inequality
{\mathfontsize{}
\begin{equation}
    n^\eta \geq \frac{\sum_{j=1}^{\ktilde} \ttilde_j(n)^{\eta}}{\ktilde}. \label{ineq1}
\end{equation}
}%
This follows from $n = \sum_{j=1}^{\ktilde} \ttilde_j(n)$ and that the function $f(x)=x^\eta$ for positive $x \in \mathbb{R_+}$ and $\frac{1}{2} \leq \eta < 1$ is concave and monotonically increasing. Hence, we have 
{\mathfontsize{}
\begin{equation}
\begin{aligned}
    f\left(\sum_{j=1}^{\ktilde} \ttilde_j(n)\right) &\geq f\left(\sum_{j=1}^{\ktilde} \frac{1}{\ktilde} \ttilde_j(n)\right) &&(\text{mon. inc.}) \\
     &\geq \frac{\sum_{j=1}^{\ktilde} f\left(\ttilde_j(n)\right)}{\ktilde}. &&(\text{concave})
\end{aligned}
\end{equation}}%

Now we establish the same bound for the term (\ref{ii}):
{\mathfontsize{}
\begin{align*}
    (\ref{ii}) &=  \mathbb{P}\!\left(
    \begin{aligned}
        \left\{\Ups - n \sum_{j=1}^{\ktilde} \ptilde_j \mutilde_j \geq n^{\eta}z\right\} \\
        \cap \bigcap_{j=1}^{\ktilde} \left\{\left|\ttilde_j(n)-n\ptilde_j\right|\leq \frac{n^\eta z}{2R\ktilde}\right\}
    \end{aligned}
    \right)  \\
    &\leq \mathbb{P}\!\left(\Ups -  \sum_{j=1}^{\ktilde} \ttilde_j(n) \mutilde_j + \ktilde \frac{n^\eta z }{2R\ktilde} R \geq n^{\eta}z\right) \romantag\label{ast1} \\
    &= \mathbb{P}\!\left(\Ups -  \sum_{j=1}^{\ktilde} \ttilde_j(n) \mutilde_j \geq \frac{n^{\eta}z}{2}\right) \\
    &\leq \mathbb{P}\!\left(\Ups -  \sum_{j=1}^{\ktilde} \ttilde_j(n) \mutilde_j \geq \sum_{j=1}^{\ktilde} \frac{\ttilde_j(n)^{\eta}z}{2\ktilde}\right)
            && (\text{by \cref{ineq1}}) \\
        &\leq \sum_{j=1}^{\ktilde} \mathbb{P}\!\left(\ttilde_j(n)\!\left(\ybar_{j, \ttilde_j(n)} - \mutilde_j\right)\! \geq  \frac{\ttilde_j(n)^{\eta}z}{2\ktilde}\right)
            &&(\text{Union Bound}) \\
    &\leq \ktilde \frac{\beta \ktilde^\xi 2^\xi }{z^\xi}
            &&(\text{Conc. of }\ybar_{j,\ttilde_j(n)}) \\
    &\leq 2\ktilde \frac{\beta 2^\xi R^\xi \ktilde^\xi}{z^\xi}.
\end{align*}}%
Here, the concentration of $\ybar_{j,\ttilde_j(n)}$ is directly applicable for $\frac{z}{2\ktilde} \geq 1$ and for $\frac{z}{2\ktilde} < 1$, it trivially follows from the right-hand side being greater than $1$, since then $\frac{(2\ktilde)^\xi}{z^\xi}>1$ and $\beta > 1$. Additionally inequality (\ref{ast1}) follows from
{\mathfontsize{}
\begin{alignat}{3}
    &&\left|\ttilde_j(n)-n\ptilde_j\right|&\leq \frac{n^\eta z}{2R\ktilde} \text{ and } \mutilde_j \in [-R,R], &&\forall j\in [\ktilde] \nonumber\\
    \Rightarrow&& -n\ptilde_j\mutilde_j &\leq -\ttilde_j(n)\mutilde_j + \frac{n^\eta z}{2R\ktilde} R,  &&\forall j\in [\ktilde] \\
    \Rightarrow&& - \sum_{j=1}^{\ktilde} n \ptilde_j \mutilde_j &\leq - \sum_{j=1}^{\ktilde} \ttilde_j(n) \mutilde_j + \ktilde \frac{n^\eta z }{2R\ktilde} R \nonumber .
\end{alignat}}%
By combining both bounds, we obtain
{\mathfontsize{}
\begin{equation}
     \mathbb{P}\!\left(n\ybar_n - n\mutilde \geq n^{\eta}z\right) \leq 2 \frac{\beta 2^{\xi+1}R^\xi \ktilde^{\xi+1}}{z^\xi} = \frac{\beta'}{z^{\xi'}},   
\end{equation}}%
where $\beta' := \beta\, 2^{\xi+2} R^\xi \ktilde^{\xi+1} > 1$ and $\xi' := \xi > 0$, $\frac{1}{2} \leq \eta':=\eta < 1$. The other direction follows analogously. \\

Note that in this case, we do not need to show convergence separately since convergence in expectation follows from convergence in probability.
\end{proof}

\begin{proof}[Proof of \Cref{thm:main}]
\hfill\newline\noindent
First, note that we can assume the same constants $\beta'$, $\xi'$, and $\eta'$ for the entire intermediate layer, i.e., $\forall i \in [K]$, since we can choose the largest $\beta'$, making every inequality hold. Thus, we let $\beta' := \beta\, 2^{\xi+2} R^\xi \max_{i\in[K]}(K_i)^{\xi+1}$. We get the desired result by applying \cref{thm:1,thm:fixedprob} successively.

By assumption, we have $\{X^i_{j,t}: t \in \mathbb{N}\}, \forall j \in [K_i], \forall i\in [K]$ satisfy the convergence and concentration properties for constants $\beta, \xi, \eta$. By application of \cref{thm:fixedprob}, it follows that $\{X_{i,t}: t \in \mathbb{N}\}, \forall i \in [K]$ satisfy the convergence and concentration properties for constants $\beta', \xi', \eta'$. Now \cref{thm:1} yields that 
$\{X_t: t \in \mathbb{N}\} $ satisfies the convergence and concentration properties for constants $ \beta'', \xi'', \eta''$, where the empirical average $\xbar_n$ concentrates around $\mu_{i^*}$, representing the converged empirical average of the best action with respect to the converged values $\{\mu_i:i\in[K]\}$.

The relationship between the constants at the root and the leaves can be summarized as follows:
{\mathfontsize{}
\begin{align}
    & \eta'' = \frac{\alpha}{\xi (1-\eta)} \label{eq:eta''}\\
    & \xi'' = \alpha - 1 \label{eq:xi''}\\
    & \beta'' = \max\!\left\{c_2, 2c_1^{\alpha-1} \max\!\left\{\beta', \frac{2(K-1)}{(\alpha-1)(1+A(N_p))^{\alpha-1}}\!\right\}\!\!\right\} \label{eq:beta''} \\
    & \beta' = \beta 2^{\xi+2}R^{\xi}\max_{i\in[K]}(K_i)^{\xi+1} 
\end{align}}%
Here, the constants $c_1, c_2, N_p$ and the function $A(t)$ were introduced by \citet{shah2020non} and are given by:%
{\mathfontsize{}%
\begin{align}%
    & c_1 = 2RK(\frac{2}{\Delta_{min}} \beta'^{\frac{1}{\xi}})^{\frac{1}{1-\eta}} \\
    & c_2 = 2R(N_p-1)^{1 - \frac{\alpha}{\xi(1-\eta)}} \\
    & A(t) = \left\lceil (\frac{2}{\Delta_{min}} \beta'^{\frac{1}{\xi}})^{\frac{1}{1-\eta}} t^{\frac{\alpha}{\xi (1-\eta)}} \right\rceil \\
    & \begin{aligned}
        N_p = & \min_{t \in \mathbb{N}} t \\
              & \;\begin{aligned}%
       \textrm{s.t.} \;& t \geq \max\{1, A(t)\} \\
                       & t \leq (2R(3 + A(t) - 4K))^\frac{1}{\eta}
              \end{aligned}%
    \end{aligned}%
\end{align}}%
\end{proof}

\subsection{Induction Over Layers}
\label{trees}
To apply the results of \cref{thm:main} to trees, we first need to show that the assumptions of \cref{thm:main} are satisfied at the leaf level of the tree, from which we can then inductively follow, by successive application of \cref{thm:main}, that the polynomial concentration is maintained throughout the entire tree up to the root. To this end, we observe the leaf level for the finite horizon $H$ as follows.   

\subsubsection{Induction Start}

At leaf level we observe \textbf{stationary} reward sequences $\{X^H_{i,t}\colon t\in \mathbb{N}\}$ that are independent and identically distributed (i.i.d.) for $i\in[K_H]$. Further, we assume the rewards are bounded within $[-R, R]$. We define their empirical mean of the first $n$ trials as $\xbar^H_{i,n}= \frac{1}{n} \sum_{t=1}^{n} X^H_{i,t}$ and its expectation $\mu^H_{i,n}:= E\left(\xbar^H_{i,n}\right)$. Then, we can state the following lemma:

\begin{lemma}
The reward sequences at the leaf level satisfy the convergence and concentration properties:
\begin{enumerate}
    \item (Convergence): For every $i\in [K_H]$ there exists $ \mu^H_i \in [-R,R]$ such that
    {\mathfontsize{}
    \begin{equation}
        \mu^H_{i,n}:=\mathop{\mathbb{E}}\left(\xbar^H_{i,n}\right)\xrightarrow{n\xrightarrow{}\infty} \mu^H_i
    \end{equation}}%
    \item (Concentration): There exist constants $\beta^H>1$, $\xi^H>0$, $\frac{1}{2} \leq \eta^H < 1$ such that for every $z\in \mathop{\mathbb{R}}$ with $z\geq 1$ and $n\in \mathop{\mathbb{N}}$ we have
    {\mathfontsize{}
    \begin{equation}
    \begin{aligned}    
        \mathbb{P}\!\left(n\xbar^H_{i,n} - n\mu^H_i \geq n^{\eta^H} z\right) \leq \frac{\beta^H}{z^{\xi^H}} \ ,\\
        \mathbb{P}\!\left(n\xbar^H_{i,n} - n\mu^H_i \leq -n^{\eta^H} z\right) \leq \frac{\beta^H}{z^{\xi^H}}.
    \end{aligned}
    \end{equation}}%
\end{enumerate}
\end{lemma}

\begin{proof}
Since the reward sequences are i.i.d., we can apply Hoeffdings inequality similar to the proof of \cref{lemma} to obtain a stronger result, which we then weaken to be in the desired form. Note that convergence follows trivially from
{\mathfontsize{}
\begin{equation}
\begin{aligned}    
    \mu^H_{i,n} &= \mathbb{E}\!\left(\xbar^H_{i,n}\right) \\
    &= \frac{1}{n} \sum_{t=1}^n \mathbb{E}\!\left(X^H_{i,t}\right) &&(\text{lin.}) \\
    &=\mathbb{E}\left(X^H_{i,1}\right)  &&(\text{i.i.d.})\\
    &=: \mu^H_{i} , \forall n\in \mathbb{N}.
\end{aligned}
\end{equation}}%
Now for $z\in \mathop{\mathbb{R}}$ with $z\geq 1$ and $n\in \mathop{\mathbb{N}}$, by applying Hoeffdings inequality, we obtain 
{\mathfontsize{}
\begin{equation}
\begin{aligned}
    &\mathbb{P}\!\left(n\xbar^H_{i,n} - n\mu^H_i \geq n^{\eta^H} z\right) \\
    & \leq e^{- \frac{2n^{2\eta^H} z^2 }{n4R^2} } && (\text{Hoeffding: i.i.d. \& } X^H_{i,t}\in [-R,R]) \\
    & = \frac{1}{e^{ \frac{n^{(2\eta^H-1)} z^2 }{2R^2} } } \\
    &\leq \frac{1}{e^{\frac{z^2}{2R^2}} }  && (2\eta^H-1 \geq 0) \\
    &\leq \frac{\beta^H}{z^{\xi^H}}.  && \!\left(\forall \beta^H \geq \max_{z \geq 1}\frac{z^{\xi^H}}{e^{\frac{z^2}{2R^2}}}\!\right)
\end{aligned}
\end{equation}}%
Hence, for any finite $\xi^H > 0$, $\frac{1}{2} \leq \eta^H < 1$ we can choose $\beta^H > 1$ large enough that the desired inequality holds. Specifically, any $\beta^H \!\!\geq\!\! R^{\xi^H} (\xi^H)^{\frac{\xi^H}{2}} e^{-\frac{\xi^H}{2}}$ will suffice.

Note that for large enough $z$ any $\beta_T>1$ suffices, since $z^{\xi^H}= \mathcal{O}\!\left(\!e^{\frac{z^2}{2R^2}}\!\right)\!$ for any finite $\xi^H>0$.
\end{proof}

\subsubsection{Induction Step}
The induction step is precisely given by \cref{thm:main}. As an induction hypothesis, we assume that some layer $h \in \{2, \dots, H\}$ satisfies the convergence and concentration properties for constants $\beta^h, \xi^h, \eta^h$ from which it follows by \cref{thm:main} that the next layer $(h-1)$ also satisfies the properties for derived constants $\beta^{h-1}, \xi^{h-1}, \eta^{h-1}$ according to \cref{eq:eta'',eq:xi'',eq:beta''} and appropriate parameter $\alpha^h > 2$ with $\xi^h\eta^h\left(1-\eta^h\right) \leq \alpha^h < \xi^h \left(1-\eta^h\right)$ . Hence, for a finite horizon with bounded rewards, we can attain polynomial concentration and convergence throughout the whole tree up to the root of the search. Note that this inductive argument is equivalent to the deterministic case \cite{shah2020non}.


\section{Experimental Details}
For our small motivational experiment, we used the FrozenLake environment, which is explained in the following subsection, with $n \in \{2^{10}, 2^{11}, 2^{12}, 2^{13}, 2^{14}\}$ over 300 random seeds. The results can be seen in \cref{fig:exp}. It displays the average discounted return observed over the different seeds, including the standard error, against the number of simulations $n$.
Further, \cref{fig:regret} shows the regret of MCTS for various $n$ in state $14$ (\cref{fig:frozenlake} left of the present) and 10 random seeds. This plot is a more direct visualization of \cref{thm:main}.
Details about hyperparameters and the computing infrastructure are given later in this section.

\subsection{The FrozenLake Environment}
The FrozenLake environment is a Gymnasium environment \cite{towers2024gymnasium}.
It offers four actions in every state: Left, Down, Right, and Up.
With a probability of $\nicefrac{1}{3}$, the agent transitions as intended. Otherwise, they will move in one of the two perpendicular directions, each with a probability of $\nicefrac{1}{3}$. The goal for the player is to reach the goal state in the lower right corner (\cref{fig:frozenlake}, present) to get a reward of $1$, and the episode terminates. It always starts in the upper left corner. If it falls into a hole, the episode also terminates, with a reward of $0$. All intermediate rewards are $0$ as well. We use the $4 \times 4$ map and allow $T = 400$ time steps until the episode terminates, and a discount factor $\gamma = 0.99$.
\begin{figure}[tb]
    \centering
    \includegraphics[width=0.35\columnwidth]{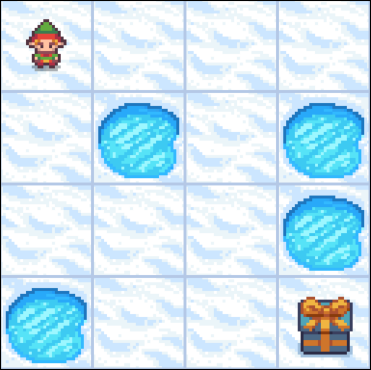}
    \caption{A screenshot of the FrozenLake environment showing the $4 \times 4$ map with the player at the start state, goal (present) and 4 holes.}
    \label{fig:frozenlake}
\end{figure}

\subsection{Hyperparameters}
Computing the exact constants is time-consuming in practice. Therefore, we chose to follow the path of \citet{shah2020non} ($\eta = \nicefrac{1}{2}$ and $\frac{\alpha}{\xi} = \nicefrac{1}{4}$). Further, we selected $\beta^{\frac{1}{\xi}} = 2$ so that
we get an exploration bonus of
{\mathfontsize{}%
\begin{equation}
    B_{t,s} = 2 \sqrt{\frac{\sqrt{t}}{s}}. 
\end{equation}}%
We strongly believe this is sufficient to show the practical value of our theoretical contribution.

\subsection{Computational Infrastructure}
Our experiments were conducted on Google Colab\footnote{\url{https://colab.research.google.com}} using the Python 3 CPU runtime.

\end{document}